\documentclass{llncs} 

\usepackage{amsmath, amssymb}
\usepackage{framed}
\usepackage{listings}
\usepackage{fancyvrb}
\usepackage{verbatim}
\usepackage{algorithm}
\usepackage[noend]{algpseudocode}
\usepackage{mathtools}
\usepackage{graphicx}
\usepackage{footnote}
\usepackage{cprotect}
\usepackage{tikz}
\usepackage{pgfplots}
\usepackage{filecontents}
\usepackage[compatibility=false]{caption}
\usepackage{subcaption}
\usetikzlibrary{shapes}
\usetikzlibrary{automata}
\usetikzlibrary{arrows}

\newcommand{\rewritesTo}{\rightsquigarrow}

\newtheorem{prop}{Proposition}

\begin{document}
\title{Augmenting Stream Constraint Programming with Eventuality Conditions}

\author{Jasper C.H. Lee\inst{1} \and Jimmy H.M. Lee\inst{2} \and Allen Z. Zhong\inst{2}}
\institute{Department of Computer Science \\Brown University, Providence,
RI 02912,
USA\\{\tt jasperchlee@brown.edu} \and Department of Computer Science
and Engineering \\The Chinese University of Hong Kong, Shatin, N.T., Hong
Kong\\{\tt \{jlee,azhong\}@cuhk.edu.hk}}

\maketitle

\vspace*{-.3cm}
\begin{abstract}

Stream constraint programming is a recent addition to the family of
constraint programming frameworks, where variable domains are sets of
infinite streams over finite alphabets.  Previous works showed promising
results for its applicability to real-world planning and control problems.
In this paper, motivated by the modelling of planning applications, we
improve the expressiveness of the framework by introducing 1) the ``until"
constraint, a new construct that is adapted from Linear Temporal Logic and
2) the \verb~@~ operator on streams, a syntactic sugar for which we provide
a more efficient solving algorithm over simple desugaring.  For both
constructs, we propose corresponding novel solving algorithms and prove
their correctness.  We present competitive experimental results on the
Missionaries and Cannibals logic puzzle and a standard path planning
application on the grid, by comparing with Apt and
Brand's method for verifying eventuality conditions using a CP approach.

\end{abstract}

\section{Introduction}

Stream constraint programming~\cite{CP_on_stream,towards_practical} is a recent addition to the family of constraint programming frameworks.
Instead of reasoning about finite strings~\cite{Golden:2003}, the domain of the constraint variables in a \emph{Stream Constraint Satisfaction Problem} (St-CSP) consists of \emph{infinite streams} over finite alphabets.
A St-CSP solver computes not only one but \emph{all} stream solutions to a given St-CSP, succinctly represented as a deterministic B\"{u}chi automaton.
Because of the infinite stream domains, and the fact we can find all solutions, the framework is particularly suitable for modelling problems involving time series, for example in control and planning, using one variable for each stream as opposed to using one variable per stream per time point in traditional finite domain constraint programming~\cite{Apt:2006}.
Lallouet et al.~\cite{CP_on_stream} first demonstrated such capabilities by implementing the game controller of Digi Invaders\footnote{See \verb~https://www.youtube.com/watch?v=1YafgAcmov4~ for a video of the game as implemented by Casio.}, a popular game on vintage Casio calculator models, using the St-CSP framework.
Lee and Lee~\cite{towards_practical} further applied the framework to synthesise PID controllers for simple robotic systems\footnote{See \verb~http://www.youtube.com/watch?v=dT56qAZt8hI~ and \verb~http://www.youtube.com/watch?v=5GvbG3pN0vY~ for video demonstrations.}.

In addition to using St-CSPs for control, Lee and Lee~also proposed a
framework for modelling planning problems as St-CSPs, adapting that of
Ghallab et al.~\cite{ghallab2016automated} for finite domain constraint
programming.  Even though the St-CSP framework can express the
entirety of what finite domain CSPs could, there are still natural
constraints on plans that we expect to be able to express but are unable
to.  For example, we cannot express the constraint that the generated plan
must \emph{eventually} satisfy a certain condition, without imposing a hard
upper bound on the number of steps before the plan must satisfy the
condition.

This paper focuses on enhancing the expressiveness of the St-CSP framework,
using planning problems as a motivation.  We introduce the ``until"
constraint (Section~\ref{Sect:until}), adapted from Linear Temporal Logic
(LTL)~\cite{Pnueli:1977}, which includes as a special case the
``eventually" constraint.  In addition, in the case where we do wish to
concretely bound the number of steps before a condition is satisfied, we
introduce the \verb~@~ operator (Section~\ref{at_operator}) to simplify the
modelling from the approach of Lee and Lee.  There are two advantages to
using the new operator in constraints: 1) we can better leverage known
structure to accelerate solving, and 2) the notation is significantly less
cumbersome, as measured in the length of the constraint expressions.  
We give experimental evidence (Section~\ref{Sect:Experiments}) of the competitiveness of our new solving algorithms.

\section{Background}
\label{background}
\vspace*{-.1cm}
We review the basics of stream constraint programming.

\vspace*{-.3cm}
\paragraph{\bf Existing Stream Expressions and Constraints.}

A \emph{stream} $a$ over a (finite) alphabet $\Sigma$ is a function $\mathbb{N}_0 \to \Sigma$.
For example, the function $a(n) = n \bmod 2$ is a stream over \emph{any} alphabet containing $\{0,1\}$.
The set of all streams with alphabet $\Sigma$ is denoted by $\Sigma^\omega$.
The notation $a(i, \infty)$ is used for the stream suffix $a'$ where $a'(j) = a(j+i)$.
For a language $L$, we similarly define $L(i, \infty) = \{ a(i, \infty) \, | \, a \in L \}$.
In this paper, we are only concerned with St-CSPs whose variables take alphabets that are integer intervals, i.e. $[m..n]^\omega$ for some $m \le n \in \mathbb{Z}$.
However, the framework generalises naturally to any other finite alphabets.

To specify expressions, there are primitives such as variable streams, which are the variables in the St-CSP, and constant streams.
For example, the stream $2$ denotes the stream $s$ where $s(i) = 2$ for all $i \ge 0$.

\emph{Pointwise} operators, such as integer arithmetic operators \{\verb-+-, \verb+-+, \verb+*+, \verb+/+, \verb+%+\}, combine two streams at each index using the corresponding operator.
Integer arithmetic relational operators are \{\verb+lt+, \verb+le+, \verb+eq+, \verb+ge+, \verb+gt+, \verb+ne+\}.
They compare two streams pointwisely and return a \emph{pseudo-Boolean stream}, that is a stream in $[0..1]^\omega$.
Pointwise Boolean operators $\{ \verb+and+, \, \verb+or+ \}$ act on any two pseudo-Boolean streams $a$ and $b$.
The final pointwise operator supported is \verb+if+-\verb+then+-\verb+else+.
Suppose $c$ is pseudo-Boolean, and $a, b$ are streams in general, then $(\verb+if+ \; c \; \verb+then+ \; a \; \verb+else+ \; b)(i)$ is $a(i)$ if $c(i) = 1$ and $b(i)$ otherwise.
There are also three \emph{temporal} operators, in the style of the Lucid programming language~\cite{Wadge:1985}: \verb+first+, \verb+next+ and \verb+fby+.
Suppose $a$ and $b$ are streams.
We have \verb+first+ $a$ being the constant stream of $a(0)$, and \verb+next+ $a$ being the ``tail" of $a$, that is \verb+next+ $a$ = $a(1, \infty)$.
In addition, $a$ \verb+fby+ $b = c$ is the concatenation of the head of $a$ with $b$ ($a$ \emph{followed by} $b$), that is $c(0) = a(0)$ and $c(i) = b(i-1)$ for $i \ge 1$.
Note that stream expressions can involve stream variables.
For example, (\verb~first~ $y$) \verb~+~ (\verb~next~ $x$) is an expression.

Given stream expressions, we can now use the following relations to express stream constraints.
For integer arithmetic comparisons $R$ $\in$ $\{$\verb+<+, \verb+<=+, \verb+==+, \verb+>=+, \verb+>+, \verb+!=+$\}$, the constraint $a \, R \, b$ is \emph{satisfied} if and only if the arithmetic comparison $R$ is true at every point in the streams.
Therefore, a constraint is \emph{violated} if and only if there exists a time point at which the arithmetic comparison is false.
For example, \verb~next~ $x$ \verb~!=~ $y$ \verb~+~ $1$ is a constraint enforcing that the stream expression $y$ \verb~+~ $1$ is not equal to the stream \verb~next~ $x$ at all time points.
Similarly, we define the constraint $a$ \verb~->~ $b$ to hold if and only if for all $i \ge 0$, $a(i) \neq 0$ implies $b(i) \neq 0$.
Here we use the C language convention for interpreting integers as Booleans.

Care should be taken to distinguish between constraints and relational expressions.
Relational operators take two streams and output a pseudo-Boolean stream.
Constraints, however, are relations on streams.
Two simple examples illustrate the difference: $x$ \verb~le~ $4$ is a pseudo-Boolean stream, whereas $x$ \verb~<=~ $4$ is a constraint that enforces $x$ to be less than or equal to $4$ at every time point.

\paragraph{\bf Stream Constraint Satisfaction Problems.}

\begin{definition}{\cite{CP_on_stream,towards_practical}}
A \emph{stream constraint satisfaction problem} (St-CSP) is a triple $P = (X, D, C)$, where $X$ is the set of variables and $D(x) = (\Sigma(x))^\omega$ is the \emph{domain} of $x \in X$, the set of all streams with alphabet $\Sigma(x)$.
A constraint $c \in C$ is defined on an ordered subset $Scope(c)$ of variables, and every constraint must be formed as specified previously (though it is the aim of this paper to extend the class of specifiable constraints).
\end{definition}

Fig.~\ref{stcsp for planning} gives an example St-CSP.
An \emph{assignment} $A : X \to \bigcup_{x \in X} D(x)$ is a function mapping a variable $x_i \in X$ to an element in its domain $D(x_i)$.
A constraint $c$ is satisfied by an assignment $A$ if and only if it is satisfied by the streams $\{A(x)\}_{x \in Scope(c)}$, and a St-CSP $P$ is \emph{satisfied} by $A$ if and only if all constraints $c \in C$ are satisfied by $A$.
We call the assignment $A$ a \emph{solution} of the St-CSP $P$.
We denote the \emph{solution set} of $P$, namely the set of all solutions $A$ to $P$, by $sol(P)$.
The St-CSP $P$ is \emph{satisfiable} if $sol(P)$ is non-empty, and \emph{unsatisfiable} otherwise.
We also say that two St-CSPs $P$ and $P'$ are \emph{equivalent} (denoted $P \equiv P'$) when $sol(P) = sol(P')$.

Given a set of constraints $C$ and an integer $i$, the \emph{shifted view} of $C$ is defined as $C(i, \infty) = \{ c_k(i, \infty) \, | \, c_k \in C \}$ by interpreting constraints as languages.
Similarly, given an St-CSP $P = (X, D, C)$ and a point $i$, the \emph{shifted view} of $P$ is defined as $\hat{P}(i) = (X, D, C(i, \infty))$.

\vspace*{-.3cm}
\paragraph{\bf Solving St-CSPs.}
\label{Sect:BackgroundSearch}

Lallouet et al.~\cite{CP_on_stream} showed that the solution set $sol(P)$
of a St-CSP $P$ is a \emph{deterministic $\omega$-regular language},
accepted by some \emph{deterministic B\"{u}chi automaton}
$\mathcal{A}$, which is a deterministic finite automaton for languages of streams~\cite{Buechi:1990}.
A stream $s$ is accepted by $\mathcal{A}$ if the execution of $\mathcal{A}$ on input $s$ visits accepting states of $\mathcal{A}$ infinitely many times.
When given a St-CSP $P$, the goal of a St-CSP solver, then, is to produce a \emph{deterministic B\"{u}chi automaton} $\mathcal{A}$, called a \emph{solution automaton} of $P$, that accepts the language $sol(P)$.
We note that the work of Golden and Pang~\cite{Golden:2003} for finite string constraint reasoning also finds all solutions as a single regular expression.

A St-CSP can be solved by a two-step
approach~\cite{CP_on_stream,towards_practical}.
First, a given St-CSP $P$ is
\emph{normalised} into some normal form $P'$ 
where auxiliary variables may
be introduced, but $P'$ is equivalent to $P$ {\em modulo\/} the auxiliary
variables.  Afterwards, the
\emph{search tree} (as defined below) is explored and ``morphed" into a
deterministic B\"{u}chi automaton via a \emph{dominance detection
procedure}, which is then output as the solution automaton.
In the rest of the paper, when we augment the language for specifying
stream expressions and constraints, we shall also follow the above two-step
approach to solve these new classes of St-CSPs.  As such, we only have to
(a) specify our new normal forms, (b) give a corresponding normalisation
procedure, and (c) detail the new dominance detection procedures.

We now define the notion of search trees for St-CSPs, adapted from that for
traditional finite-domain CSPs~\cite{Dechter:2003}.  We also describe how
they are explored and how dominance detection allows us to compute solution
automata from search trees.  A \emph{search tree} for a St-CSP $P$ is a
tree with potentially infinite height.  Its nodes are St-CSPs with the root
node being $P$ itself.  The \emph{level} of a node $N$ is defined as $0$
for the root node and recursively for descendants.  A child node $Q' = (X,
D, C \cup \{c'\})$ at level $k+1$ is constructed from a parent node $P' =
(X, D, C)$ at level $k$ and an \emph{instantaneous assignment} $\tau(x) \in
\Sigma(x)$, where $\tau$ takes a stream variable $x$ and returns a value in
$\Sigma(x)$.  In other words, $\tau$ assigns a value to each variable at
time point $k$.  The constraint $c'$ specifies that for all $x \in X$,
$x(k) = \tau(x)$ and for all $i \neq k$, $x(i)$ is unconstrained.  We write
$P' \overset{\tau}{\to} Q'$ for such a parent to child construction, and
label the edge on the tree between the two nodes with $\tau$.
During search in practice, we shall \emph{not} consider every possible
instantaneous assignment, but instead consider only the ones remaining
after applying \emph{prefix-$k$ consistency}~\cite{CP_on_stream}.

We can identify a search node $Q$ at level $k$ with the shifted view $\hat{Q}(k)$.
Taking this view, if $\hat{P}(k) = (X, D, C)$ is the parent node of $\hat{Q}(k+1)$, then $\hat{Q}(k+1) = (X, D, C \cup \{c'\})(1) = (X, D, (C \cup \{c'\})(1, \infty))$ where $c'$ is the same constraint as defined above.

Recall that a constraint violation requires only a single time point at which the constraint is false.
Therefore, we can generalise the definition of constraint violation such that a finite prefix of an assignment can violate a constraint.
A sequence of instantaneous assignments from the root to a node is isomorphic to a finite prefix of an assignment, and so the definition again generalises.
Suppose $F = (X, D, C)$ is a node at level $k$ such that $\{\tau_i\}_{i \in [1..k]}$ is the sequence of instantaneous assignments that constructs $F$ from the root node, i.e. $P \overset{\tau_1}{\to} \ldots \overset{\tau_k}{\to} F$.
We say node $F$ is a \emph{failure} if and only if $\{\tau_i\}_{i \in [1..k]}$ violates a constraint $c \in C$.

Given a normalised St-CSP $P$, its search tree is then explored using depth first search.
Backtracking happens when the current search node is a failure.
A search node $M$ at level $k$ is said to \emph{dominate} another search node $N$ at level $k'$, written $N \prec M$, if and only if their shifted views are equivalent ($\hat{M}(k) \equiv \hat{N}(k')$) and $M$ is visited before $N$ during the search~\cite{CP_on_stream,towards_practical}.
When the algorithm visits a search node $N$ that is dominated by a previously visited node $M$, the edge pointing to $N$ is redirected to $M$ instead.
If the algorithm terminates, then the resulting (finite) structure is a deterministic B\"{u}chi automaton (subject to accepting states being specified).
If dominance detection were perfect, then the search algorithm terminates, because every branch either ends in a failure or contains two nodes with the same shifted views~\cite{CP_on_stream}.
The crucial missing detail from this high-level algorithm, then, is \emph{exactly} how dominance is \emph{detected} in practice.
Search node dominance is an inherently semantic notion, implying that it is often inefficient to detect precisely.
Thus, previous works identify efficient \emph{syntactic approximations} to detecting dominance such that the overall search algorithm terminates~\cite{CP_on_stream,towards_practical}.
We shall also give a new dominance detection procedure in light of the new ways of forming stream expressions and constraints.
As for specifying the set of accepting states, previous work take \emph{all} states as accepting states, whereas we shall give a more nuanced criterion.

\section{The ``Until" Constraint}
\vspace*{-.1cm}
\label{Sect:until}
In this section, we introduce the ``until" constraint to the St-CSP framework.
Recall that all the stream constraints introduced in Section~\ref{background} are pointwise predicates.
That is, the constraint is satisfied if its corresponding predicate holds for every single time point of its input streams.
The ``until" constraint, as we shall later see, is \emph{not} a pointwise constraint.

Let us consider the following path planning problem on the standard $n \times n$ grid world domain~\cite{Sturtevant:2012,Harabor:2011}.
Between any two neighbouring vertices on the grid, there could be 0, 1 or 2 \emph{directed} edges.
We ask for all paths on the directed graph from a given start point that eventually visit a given end point.

Our method finds more than a shortest path.
Modelling this problem as a St-CSP allows us to find a succinct description of \emph{all} the paths, and moreover allows for additional side constraints.
Well-studied side constraints in the literature include precedence constraints~\cite{Kilby:2000} and time window constraints~\cite{Pesant:1998}.

We can formulate as a St-CSP the condition that the path starts at $(i_s,j_s)$, has to respect the graph, and furthermore in the St-CSP model check whether the goal of visiting the end point $(i_g,j_g)$ is attained.
This St-CSP is shown in Fig.~\ref{stcsp for planning}.
We use variables $x, y$ to represent the $x$ and $y$ coordinates of the current position.
In addition, a variable $goal$ denotes if we have visited the end point.
The second to last constraint is such that if $goal$ is true in one time point, it stays true in the next one as well.
The last constraint says that if the path has reached the end point, then it stays there indefinitely.

\begin{figure}[t]
\small
\vspace*{-.3cm}
\begin{tabular}{l}
var $x,y$ with alphabet $[1..n]$\\
var $goal$ with alphabet $[0..1]$\\
\\
\verb~first~ $x$ \verb~==~ $i_s$\\
\verb~first~ $y$ \verb~==~ $j_s$\\
\\
For each vertex $(i, j)$,\\
((\verb~next~ $x$ \verb~eq~ $i$) \verb~and~ (\verb~next~ $y$ \verb~eq~ $j$)) \verb~->~ (($x$ \verb~eq~ $i$ \verb~and~ $y$ \verb~eq~ $j$) \verb~or~ \\
\hspace*{.2cm} ($x$ \verb~eq~ $i_1$ \verb~and~ $y$ \verb~eq~ $j_1$) \verb~or~ $\ldots$ \verb~or~ ($x$ \verb~eq~ $i_{d_{(i,j)}}$ \verb~and~ $y$ \verb~eq~ $j_{d_{(i,j)}}$) \\
where $(i_1, j_1), \ldots, (i_{d_{(i,j)}}, j_{d_{(i,j)}})$ have edges into $(i, j)$\\
and $d_{(i,j)}$ is the in-degree of $(i,j)$\\
\\
$goal$ \verb~==~ ($x$ \verb~eq~ $i_g$ \verb~and~ $y$ \verb~eq~ $j_g$) \verb~or~ (0 \verb~fby~ $goal$)\\
$goal$ \verb~eq~ 1 \verb~->~ (($x$ \verb~eq~ \verb~next~ $x$) \verb~and~ ($y$ \verb~eq~ \verb~next~ $y$))
\end{tabular}
\caption{St-CSP Model for the Path Planning Problem}
\label{stcsp for planning}
\vspace*{-.6cm}
\end{figure}

In this current model, we have not enforced that the goal is indeed \emph{eventually} attained at some point.
An undesirable solution to the St-CSP would be, for example, to stay in one location forever.
However, variants of the ``eventually" constraint is not expressible in the St-CSP framework prior to this work, since all constraints are inherently \emph{pointwise}.
Temporal operators are not expressive enough for our purpose, since these
operators shift streams by a constant number of time points only.
The ``eventually" constraint, on the other hand, can be satisfied at an unbounded number of time points away into the future.

We thus introduce the ``until" constraint, adapted from Linear Temporal Logic (LTL)~\cite{Pnueli:1977} and essentially equivalent to ``eventually"~\cite{Emerson:1990}.
\begin{definition}[The ``Until" Constraint]
Given two streams $a, b$, the constraint $a$ \verb~until~ $b$ is satisfied if and only if there exists a time point $i \geq 0$ such that
1) for all $j < i$, $a(j) \neq 0$ and
2) $b(i) \neq 0$.
We say that the constraint is \emph{finally satisfied} at time point $i$ if $b(i) \neq 0$.
Note that we are again adapting the C language convention for interpreting integers as Booleans.
\end{definition}

The ``eventually" constraint is expressible in terms of the ``until" constraint.
Suppose we want to express the constraint that a predicate $G$ on stream elements eventually holds, for example if $G$ is ``$goal$ \verb~eq~ 1".
Then, we can express the constraint as ``$1$ \verb~until~ $G$", or in our particular example, ``$1$ \verb~until~ ($goal$ \verb~eq~ 1)".
Conversely, ``$a$ \verb~until~ $b$" is equivalent to ``$c$ \verb+==+ $b$
\verb+fby+ (\verb+next+ $b$ \verb+or+ $c$); (\verb+not+ $c$) \verb+->+ ($a$
\verb+ne+ $0$); \verb+eventually+ $b$;".

\subsection{Normalising ``Until" Constraints}

In light of the ``until" constraint, we
give a new constraint normal form.
A St-CSP is in \emph{normal form} if it contains only constraints of the following forms:
\begin{itemize}
\item Primitive next constraints: $x_i$ \verb~==~ \verb~next~ $x_j$
\item Primitive until constraints: $x_i$ \verb~until~ $x_j$
\item Primitive pointwise constraints with no \verb~next~, \verb~fby~ or \verb~until~ (but can contain \verb~first~ operators).
\end{itemize}

Any St-CSP can be transformed into this normal form by applying the rewriting system below.
We adopt notations from programming language semantics theory
\cite{Winskel:1993}, writing $c$\,[\_] for \emph{constraint contexts}, i.e. constraints with placeholders for syntactic substitution.
For example, if $c$\,[\_] $=$ [\_ \verb~+~ $ 3$ \verb+>=+ 4], then $c$\,[\verb+first+ $\alpha$] $=$ [(\verb~first~ $\alpha$) \verb-+- 3 \verb+>=+ 4].
We also write a constraint rewriting transition as ($C_0$, $C_1$) $\rightsquigarrow$ ($C'_0$, $C'_1$), where $C_0$, $C_1$, $C'_0$ and $C'_1$ are sets of constraints.
$C_0$ is the set of constraints that potentially could be further
normalised, and $C_1$ is the set that is already in normal form.
Hence, the initial constraint pair for the St-CSP $($$X$, $D$, $C$$)$ is $(C, \{\})$.
Rules are applied \emph{in arbitrary order} until none are applicable.

\begin{itemize}
   \item $($$C_0$ $\cup$ \{$c$\,[\verb+next+ $expr$]\}, $C_1$$)$ $\rewritesTo$ $($$C_0$ $\cup$ \{$c$\,[$x_1$], $x_2$ \verb+==+ $expr$\}, $C_1$ $\cup$ \{$x_1$ \verb+==+ \verb+next+ $x_2$\}$)$, where $x_1$ and $x_2$ are fresh auxiliary stream variables.

   \item $($$C_0$ $\cup$ \{$c$\,[$expr_1$ \verb+fby+ $expr_2$]\}, $C_1$$)$ $\rewritesTo$ $($$C_0$ $\cup$ \{$c$\,[$x_1$], $x_2$ \verb+==+ $expr_1$, $x_3$ \verb+==+ $expr_2$\}, $C_1$ $\cup$ \{\verb+first+ $x_1$ \verb+==+ \verb+first+ $x_2$, $x_3$ \verb+==+ \verb+next+ $x_1$\}$)$, where $x_1$, $x_2$ and $x_3$ are fresh auxiliary stream variables.

\item $($$C_0$ $\cup$ \{$expr_1$ \verb~until~ $expr_2$\}, $C_1$$)$ $\rewritesTo$ $($$C_0$ $\cup$ \{$x_1$ \verb~==~ $expr_1$, $x_2$ \verb+==+ $expr_2$\}, $C_1$ $\cup$ \{$x_1$ \verb+until+ $x_2$\}$)$, where $x_1$ and $x_2$ are fresh auxiliary stream variables.
\end{itemize}
We can check easily the following properties of the new rewriting system.
 
 \begin{prop}
 	The new rewriting system always terminates, regardless of the order in which the rules are applied.
 \end{prop}
 
 \begin{proof}
 The number of \verb~next~, \verb~fby~ and \verb~until~ keywords in the set $C_0$ is monotonically decreasing.
 \end{proof}

 \begin{prop}
 	The rewriting system has the Church-Rosser property (up to auxiliary variable renaming).
 \end{prop}
 
 \begin{proof}
 	The rules are pairwise commutative.
 \end{proof}

 \begin{prop}
 	The rewriting system is sound, in the sense that it preserves the projection of the solution set of the resulting St-CSP into the original variables.
 \end{prop}
 
 \begin{proof}
 	By induction on the number of rule applications, since the statement holds every time a rule is applied.
 \end{proof}

\vspace*{-.2cm}
\subsection{Search Algorithm and Dominance Detection}
\vspace*{-.1cm}

In the following, we assume that all given St-CSPs are in normal form.

Recalling the high-level solving algorithm in Section~\ref{background}, we give in this section a concrete instantiation of the syntactic dominance detection procedure.
Our syntactic procedure should possess two key properties.
First, the procedure should be \emph{sound}:\ if two search nodes are claimed to have equivalent shifted views by the procedure, then they do indeed have equivalent shifted views.
Second, the approximation should be sufficiently close to the semantic notion, such that the overall search algorithm terminates and produces a \emph{finite} structure.
Otherwise, in the extreme scenario where the dominance detection procedure never reports any dominance, the search algorithm will simply search the entire (usually infinite) search tree, resulting in non-termination.

Our dominance detection procedure, as with previous works~\cite{CP_on_stream,towards_practical}, involves keeping
track of a \emph{syntactic} representation of the shifted view of each
search node, and detects dominance by checking \emph{syntactic equivalence}
between the two representations.  Hereafter, we refer to search nodes and
their syntactic representations interchangeably for narratory simplicity.
Each search node, then, is represented by two components: 1) a set $C$ of
St-CSP constraints and 2) a table $h$, called
\emph{historic values}, storing for each variable $x_j$ in a primitive next
constraint ``$x_i$ \verb~==~ \verb~next~ $x_j$" the value assigned to $x_i$
\emph{at the previous time point}.  The historic values are used to enforce
primitive next constraints.  If a value $v$ is assigned to $x_i$ at the
previous time point, then \verb~first~ $x_j$ \verb~==~ $v$ holds in the
shifted view of the current search node.  We thus store $v$ in the table
entry for $x_j$.

Algorithm~\ref{node_construct_alg_until} gives pseudocode for two functions, \textsc{Construct} and \textsc{AreEqual}, both adapted from the algorithm of Lee and Lee~\cite{towards_practical} with \emph{minimal} changes (lines 6--8) to accommodate ``until" constraints.
The function \textsc{Construct} takes a parent search node $\hat{P}(k)$ and an instantaneous assignment $\tau$, and outputs the corresponding child search node $\hat{Q}(k+1)$ (the new constraint set $C'$ and new historic values $h'$).
The function \textsc{AreEqual}, on the input of two search nodes, just checks whether their components are syntactically equal.

\begin{algorithm}[t]
	\caption{Dominance Detection with Until Constraints}
	\label{node_construct_alg_until}
	\begin{algorithmic}[1]
		
		\Function{Construct}{Search Node $\hat{P}(k) = (C, h)$, Instantaneous Assignment $\tau$}
		\State Historic values $h'$ $\gets$ $\emptyset$
		\ForAll{primitive next constraints $x_i$ \verb~==~ \verb~next~ $x_j$}
		   \State $h'(x_j)$ $\gets$ $\tau(x_i)$
		\EndFor
		
		\State Constraint Set $C'$ $\gets$ $\emptyset$
		\ForAll{primitive until constraints $x_i$ \verb~until~ $x_j$}
		      \If{$\tau(x_j)$ = 0}
		         \State $C'$ $\gets$ $C' \cup \{\text{$x_i$ \verb~until~ $x_j$}\}$
		      \EndIf
		\EndFor
		\ForAll{primitive pointwise, next constraints $c$}
		   \State Constraint $c'$ $\gets$ $c$ evaluated with $\tau$
		   \If{$c'$ is not a zeroth order tautology}
		      \State $C'$ $\gets$ $C' \cup \{c'\}$
		   \EndIf
		\EndFor
		
		\State \Return $\hat{Q}(k+1) = (C', h')$
		\EndFunction

		\Function{areEqual}{Search Nodes $\hat{P}(k) = (C_P, h_P)$, $\hat{Q}(k') = (C_Q, h_Q)$}
		   \State \Return $(C_P = C_Q) \wedge (h_P = h_Q)$
		\EndFunction
	\end{algorithmic}
\end{algorithm}

We describe the function \textsc{Construct} in more detail.
The new set of historic values $h'$ is conceptually simple to compute.
For each primitive next constraint ``$x_i$ \verb~==~ \verb~next~ $x_j$", we store $h'(x_j) = \tau(x_i)$ where $\tau$ is the instantaneous assignment given for the construction of the child search node.
The new constraint set $C'$ is computed from $C$ by processing each constraint individually:
1) For primitive next constraints, we keep them as is and put them into $C'$.
2) For primitive pointwise constraints, we follow Lee and Lee~\cite{towards_practical} in \emph{evaluating} them using the instantaneous assignment $\tau$.
That is, we substitute every variable stream $x$ appearing in an expression whose outermost operator is the \verb~first~ operator, using the value $\tau(x)$.
This process produces expressions that consist entirely of constant
streams, pointwise operators and \verb~first~ operators, and thus can be evaluated into a single constant stream.
If, as a result, a primitive pointwise constraint becomes a numerical tautology (e.g.~1 \verb~==~ 1), we discard such a constraint.
3) For primitive until constraints ``$x_i$ \verb~until~ $x_j$" (lines 6--8), we simply check whether $\tau(x_j)$ is 1, namely if the constraint is satisfied by the instantaneous assignment $\tau$.
If so, we discard the constraint; otherwise we keep it in $C'$.

When the search algorithm terminates, which provably happens as we shall
state later, we have a finite automaton whose states have to be labelled as accepting or non-accepting.
We choose the set of accepting states as those whose constraint set $C$ contains \emph{no} primitive until constraints.
As a special case, when the given St-CSP has no ``until" constraints, then all the states are accepting.

We stress again that our algorithm requires minimal changes from previous work to support the use of ``until" constraints in St-CSPs.
The only changes we have are lines 6--8 for the treatment of primitive until constraints, as well as how we pick the set of accepting states.

We first show that the dominance detection procedure is sound.
To do so, we show that from a parent search node $(C, h)$ and an instantaneous assignment $\tau$, \textsc{Construct} computes a child node $(C', h')$ representing the correct shifted view.
Thus, if two search nodes are syntactically equivalent, the corresponding shifted views must also be equivalent.

\begin{theorem}[Soundness of dominance detection]
\label{soundness of dominance detection}
Suppose the constraint set $C'$ of the shifted view of child node $\hat{Q}(k+1)$ is output by \textsc{Construct} from the constraint set $C$ of the parent node $\hat{P}(k)$ and the instantaneous assignment $\tau_k$.
Then $sol(C' \cup \{c_2\}) = sol(\{c \cap \pi_{Scope(c)}(c_1) \, | \, c \in C\}(1, \infty))$ where $c_1$ is the constraint stating $x(0) = \tau_k(x)$ for all streams $x$, and $c_2$ is the constraint stating $x_j(0) = \tau_k(x_i)$ for all constraints $x_i$ \verb+==+ \verb+next+ $x_j$ in $C$ (and hence $C'$).
Note that $c_2$ is enforced by the set of historic values $h'$ produced by \textsc{Construct}.
\end{theorem}

\begin{proof}
	$sol(C' \cup \{c_2\}) \subseteq sol(\{c \cap \pi_{Scope(c)}(c_1) \, | \, c \in C\}(1, \infty))$:
	Consider an arbitrary solution $s \in sol(C' \cup \{c_2\})$.
	We need to show that the stream $s'$ $=$ $\tau_k$ \verb~fby~ $s$ (in a slight abuse of notation) is a solution in $sol(\{c \cap \pi_{Scope(c)}(c_1) \, | \, c \in C\})$.
	Since $s'$ clearly satisfies the constraint $c_1$, we show that $s'$ satisfies every other constraint $c \in C$.
	We proceed by performing a case analysis.
	If $c$ is a primitive next constraint, then $s'$ satisfies $c$ by virtue of $c$ being in $C'$, and that $s'$ satisfies the constraint $c_2$.
	If $c$ is a primitive until constraint, then either 1) $\tau_k$ finally satisfies $c$ or 2) $C'$ contains $c$.
	In both cases, $c$ is satisfied by $s'$.
	Lastly, if $c$ is a primitive pointwise constraint, consider the constraint $c'$ that is evaluated using the instantaneous assignment $\tau_k$.
	By construction, $c'$ is the same as $c \cap \pi_{Scope(c)}(c_1)$, and so $s'$ satisfies $c$.
	
	$sol(\{c \cap \pi_{Scope(c)}(c_1) \, | \, c \in C\}(1, \infty)) \subseteq sol(C' \cup \{c_2\})$:
	Consider an arbitrary solution $s \in sol(\{c \cap \pi_{Scope(c)}(c_1) \, | \, c \in C\})$.
	We need to show that the stream \verb~next~ $s$ is in $sol(C' \cup \{c_2\})$.
	Since $s$ satisfies all primitive next constraints, \verb~next~ $s$ satisfies $c_2$.
	Now we show, again by a case analysis, that \verb~next~ $s$ also satisfies all constraints $c' \in C'$.
	If $c'$ is a primitive next constraint or a primitive until constraint, then it is also in $C$, and so $s$ and \verb~next~ $s$ satisfy $c'$.
	Lastly, if $c'$ is a primitive pointwise constraint, then again it was evaluated from a primitive pointwise constraint $c \in C$ using $\tau_k$, such that $c' = c \cap \pi_{Scope(c)}(c_1)$.
	Therefore, $s$ and thus \verb~next~ $s$ satisfy $c'$.
\end{proof}

Having analysed the dominance detection algorithm, we can leverage the
results to prove termination and soundness of the overall search algorithm.
\begin{theorem}[Termination]
\label{termination}
	Using this new dominance detection procedure, the search algorithm always terminates.
\end{theorem}

\begin{proof}
There is a finite number of variables, each with a finite alphabet.
Therefore, there is only a finite number of possible sets of historic values.
Furthermore, since each constraint can only be evaluated at most once using finitely many possible values, there are only finitely many constraint sets that can be produced during the search algorithm.
Overall, the search algorithm can produce only finitely many syntactically distinct search nodes, and thus always terminates.
\end{proof}

\begin{theorem}[Soundness and Completeness]
\label{search soundness and completeness}
The resulting solution automaton $\mathcal{A}$ accepts the same language $L(\mathcal{A})$ as the solution set $sol(P)$ of the input St-CSP $P$.
\end{theorem}

\begin{proof}
If $P$ does not contain any (primitive) until constraints, then the theorem reduces to results in previous work.
Therefore, for the rest of the analysis, we assume that there is at least one primitive until constraint in $P$.

	$L(\mathcal{A}) \subseteq sol(P)$:
	Consider an arbitrary stream $s \in L(\mathcal{A})$.
	As the search algorithm always backtracks when it encounters a failure node, namely when a primitive pointwise constraint is violated, $s$ must satisfy all the primitive pointwise constraint in $P$.
	Furthermore, the search nodes generated by the search algorithm always respects primitive next constraints by keeping and checking the set of historic values at each search node, and so $s$ also satisfies all primitive next constraints in $P$.
	For primitive until constraints, since $s$ is accepted by $\mathcal{A}$, $\mathcal{A}$ must visit an accepting state during its run on input $s$.
	Consider the first such index $i$ of $s$.
	By the definition of \textsc{Construct}, primitive until constraints are never added to a child node, and are removed only when the given constraint is finally satisfied.
	Thus, for each primitive until constraint $c$ in $P$, there must exists an index $j_c \le i$ such that $c$ is finally satisfied by $s$ at $j_c$.
	
	$sol(P) \subseteq L(\mathcal{A})$:
	Consider an arbitrary solution stream $s \in sol(P)$.
	It corresponds to an infinite branch of the search tree.
	Since $\mathcal{A}$ was constructed from the search tree with sound dominance detection, we can show by induction that for every index $i$, there must be a (unique) state $s_i \in \mathcal{A}$ whose associated shifted view includes the suffix stream $s(i, \infty)$.
	Such sequence of states $s_i$ is a run on $\mathcal{A}$, and corresponds to the stream $s$.
	Furthermore, since $s$ is a solution to $P$, for every primitive until constraint $c$ in $P$ there must be an index $j_c$ such that $s$ finally satisfies $c$ at index $j_c$.
	Take the index $j = \max \{ j_c \}$, which exists since there are only finitely many constraints in $P$.
	It must be the case that $s_j$ is an accepting state, and in fact for every time point $k \ge j$, $s_k$ is an accepting state by the construction of \textsc{Construct} by induction.
	Thus $s \in L(\mathcal{A})$.
\end{proof}

\vspace*{-.4cm}
\subsection{Automaton Pruning}
\label{Sect:Prune}

As a post-processing step, we prune all states that cannot reach any
accepting states via a flood-fill algorithm taking time linear in the size
of the automaton (before pruning), which 
retains the accepted language by the following lemma.

\begin{lemma}
	Given a solution automaton $\mathcal{A}$, let $\mathcal{A'}$ be
	obtained from $\mathcal{A}$ by removing all states not reaching any
	accepting states.
	Then $L(\mathcal{A}) = L(\mathcal{A'})$.
\end{lemma}

\begin{proof}
	$L(\mathcal{A}) \subseteq L(\mathcal{A'})$:
	Suppose on the contrary there is a string $a \in L(\mathcal{A})$ that is no longer in $L(\mathcal{A'})$.
	Then running $\mathcal{A}$ on input $a$ must at some point reach a pruned state, namely a state that cannot reach any accepting state.
	However, since that state cannot reach any accepting state, $a$ must not be in $L(\mathcal{A})$.
	
	$L(\mathcal{A'}) \subseteq L(\mathcal{A})$:
	Every state in $\mathcal{A'}$ is a state in $\mathcal{A}$.
\end{proof}

Furthermore, the pruning gives us the following guarantee about finite runs of the resulting automaton.
\begin{theorem}
\label{Prop:Prefix}
	For any finite-length run of the generated and pruned solution automaton $\mathcal{A}$, corresponding to a finite string (stream prefix) $p$, there exists a solution stream $s \in L(\mathcal{A})$ such that $p$ is the prefix of $s$ of length $|p|$.
\end{theorem}

\begin{proof}
	Consider an arbitrary finite string $p$ that corresponds to a finite-length run of $\mathcal{A'}$, ending in some state.
	By definition, such a state can reach an accepting state, we thus extend $p$ to such that $\mathcal{A'}$ ends up in that accepting state.
	We repeat this construction indefinitely to get a stream $s$.
	By construction, executing $\mathcal{A'}$ on $s$ will visit accepting states infinitely often, and so $s$ is in $L(\mathcal{A'})$.
\end{proof}

Intuitively, the theorem says that, no matter how we run the automaton, we can always extend the (finite) stream prefix \emph{generated so far} into an infinite-length solution stream.
This therefore also guarantees that it is \emph{sound} to generate solution streams by running the automaton.

We emphasise that this pruning is for soundness, not solving efficiency.

\cprotect\section{The \verb~@~ Operator}\label{at_operator}
\vspace*{-.1cm}

With the introduced ``until" constraint along with a new solving algorithm, we can now model in St-CSPs conditions that need to be \emph{eventually} satisfied.
However, eventuality constraints might not be suitable for all application scenarios.
It could be vital to be able to impose a strict upper bound on when a condition is satisfied, whilst with an eventuality constraint, the time at which a specified condition is satisfied could be arbitrarily far into the future.

Lee and Lee~\cite{towards_practical} propose using a constraint of the form ``\verb~first next~ $\cdots$ \verb~next~ $goal$ \verb~==~ 1" to model this bound, reflected by the number of \verb~next~ operators in the constraint as the time bound.
There are, however, two disadvantages to this approach.
First, such a constraint has its own structure that we could not exploit to improve solving if we were to simply use the above syntax and current solving algorithms.
Second, the notation is cumbersome, with the length of the constraint scaling linearly with the upper bound we wish to impose.
To remedy these two issues, we propose a new temporal operator ``\verb~@~" that acts as syntactic sugar, and further give another modification to the solving algorithm (more concretely, the dominance detection algorithm) to solve constraints involving the \verb~@~ operator efficiently.
We note however that, since the \verb~@~ operator is simply a sugar, it does not enhance the expressiveness of the St-CSP framework.

\begin{definition}[The \verb~@~ operator]
	Given a stream $x$ (where $x$ is instantiated or is some expression even involving stream variables) and a \emph{number} $t \ge 1$, the stream $x$\verb~@~$t$ is defined as the constant stream $(x$\verb~@~$t)(i) = x(t)$ for all $i \ge 0$.
	Equivalently, it is defined as \verb~first next~ $\cdots$ \verb~next~ $x$, where there are $t$ many \verb~next~ operators.
\end{definition}

We require that, for the purpose of this paper, the \verb~@~ operator to take only a concrete number, instead of a variable, for its second parameter $t$.
Our solving algorithm relies crucially on this assumption.

\vspace*{-.3cm}
\subsection{Modified Constraint Normalisation}

We first augment the constraint normal form to allow for primitive \verb~@~ constraints: $x_i$ \verb~==~ $x_j$\verb~@~$t$, where $t \ge 1$.

Accordingly, we add the following rewriting rule to the constraint rewriting system presented in Section~\ref{Sect:until}.
\begin{itemize}
	\item $($$C_0$ $\cup$ \{$c$\,[$expr$\verb~@~$t$]\}, $C_1$$)$ $\rewritesTo$ $($$C_0$ $\cup$ \{$c$\,[$x_1$], $x_2$ \verb+==+ $expr$\}, $C_1$ $\cup$ \{$x_1$ \verb+==+ $x_2$\verb~@~$t$\}$)$, where $x_1$ and $x_2$ are fresh auxiliary stream variables.
\end{itemize}

This new rewriting system is also terminating, Church-Rosser and sound.
The proofs are essentially identical to those in Section~\ref{Sect:until}.

\vspace*{-.3cm}
\subsection{Changes to Dominance Detection}

Having introduced the \verb~@~ operator, we adapt the function \textsc{Construct} by describing how primitive \verb~@~ constraints are modified when we construct a child search node from its parent.
Given a primitive \verb~@~ constraint ``$x_i$ \verb~==~ $x_j$\verb~@~$t$" from a parent node, we consider two cases.
\begin{itemize}
	\item If $t > 1$, then we include ``$x_i$ \verb~==~ $x_j$\verb~@~$(t-1)$" in the new constraint set.
	\item If $t = 1$, then we include ``$x_i$ \verb~==~ \verb~first~ $x_j$" instead.
\end{itemize}

This modification is orthogonal to those for the ``until" constraint.
This new dominance detection procedure (namely \textsc{Construct} and \textsc{AreEqual}) is again sound, and induces a terminating, sound and complete overall search algorithm.
The proofs are again essentially same as those in Section~\ref{Sect:until}.

\section{Experimental Results}
\label{Sect:Experiments}

We performed experiments in two settings to demonstrate the competitiveness of our approaches: 1) solving the Missionaries and Cannibals logic puzzle and 2) solving a standard path planning problem on grid instances.
For each setting, we solve for plans that eventually attain the goal using the ``until" constraint in the model, as well as for bounded-length plans using the \verb~@~ operator.

For the ``until" experiments, we compare our approach to a standard CP approach proposed by Apt and Brand~\cite{Apt:2006}.
Their approach creates a \emph{series} of finite domain CSPs, each corresponding to a finite horizon into the future, asking if the eventuality condition is satisfiable within the horizon.
The time bound is incremented until the resulting CSP becomes satisfiable.
(The idea was also used by van Beek and Chen~\cite{vanBeek:1999}, who credit Kautz and Selman~\cite{Kautz:1992}.)
As a result, if there is no upper bound a-priori on the minimum length of successful plans, this approach may not terminate.
However, in the two settings we consider, such upper bounds do exist, and so we also experimented on using a CP solver to solve for a plan of exactly that length at the upper bound.

For the bounded-length plans scenario, we compare the use of the \verb~@~ operator to the use of the \verb~first next~ $\cdots$ \verb~next~ operator phrase, as well as to using a standard CP approach of solving the corresponding finite domain CSP.

All our experiments were run on an Intel Xeon CPU E5-2630 v2 (2.60GHz) machine with 256GB of RAM, with a timeout of 600 seconds.
We used Gecode v6.0.0 as our finite domain CP solver.
We also configured both the St-CSP solver and Gecode to \emph{not} output the solutions to the file system, so as to minimise the impact of file I/O on time.
The Gecode solver selects variables using the input order and 
according to the time point, which is the same as how the St-CSP solver
label stream variables.  Values are assigned the min value first.
We tried fail-first for Gecode, but the results are less competitive.

\subsection{Missionaries and Cannibals}
	In the Missionaries and Cannibals problem, there are $n$ missionaries and $n$ cannibals trying to cross a river from one bank to another, using a boat of capacity $b$ people.
	There are three constraints in this problem:
	1) at any time, there could be at most $b$ people on the boat,
	2) there must be at least one person on the boat on every trip and
	3) for each bank, if there are any missionaries, then the cannibals cannot outnumber the missionaries; otherwise the missionaries will perish.
	The success condition is when everyone ends up on the other bank.
	For the actual constraint model, please refer to Appendix~\ref{Appx:Missionaries}.
	
	Table~\ref{until_missionaries_and_cannibals} shows the experimental results, when we solve using the St-CSP solver for \emph{all} valid transportation plans that \emph{eventually} attains the goal.
	Rows and columns in the table give different values of $n$ and $b$ respectively.
	Each entry in the table denotes the solving time in seconds for the test case.
	The results show that our solver is able to solve the problem for reasonably large instances without suffering from exponential increases in runtime.

\begin{table}[t]
		\centering
		\footnotesize
		\caption{Missionaries and Cannibals: ``until"}
		\begin{tabular}{| c | c | c | c | c | c |}
			\hline
			& $b = 4$ & $b = 5$ & $b = 6$ & $b = 7$ & $b = 8$ \\
			\hline
			$n = 40$ & 1.456 & 1.939 & 2.307 & 2.537 & 2.959\\
			\hline
			$n = 60$ & 4.459 & 5.831 & 7.417 & 9.081 & 10.698\\
			\hline
			$n = 80$ & 9.979 & 13.45 & 17.324 & 21.356 & 26.229\\
			\hline
			$n = 100$ & 19.053 & 26.044 & 33.747 & 42.16 & 53.112\\
			\hline
			$n = 120$ & 33.56 & 44.782 & 59.113 & 73.335 & 91.351\\
			\hline
			$n = 140$ & 51.623 & 70.666 & 92.744 & 118.407 & 146.325\\
			\hline
			$n = 160$ & 76.532 & 105.341 & 139.212 & 175.149 & 219.134\\
			\hline
			$n = 180$ & 110.122 & 149.741 & 196.743 & 250.56 & 313.35\\
			\hline
			$n = 200$ & 150.137 & 207.466 & 274.537 & 348.243 & 436.469\\
			\hline
			$n = 220$ & 201.308 & 277.219 & 363.592 & 463.509 & --\\
			\hline
			$n = 240$ & 259.773 & 360.413 & 474.005 & -- & --\\
			\hline
		\end{tabular}
		\vspace*{-.2cm}
		\label{until_missionaries_and_cannibals}
	\end{table}
	
We also performed experiments using the Apt and Brand framework~\cite{Apt:2006} that uses traditional finite domain CP solvers.
\emph{\bf Such CP approach timed out on all these instances.}
On the other hand, for this particular problem there is, in fact, an upper bound on the number of steps of $n(b+1)$ if a feasible plan exists.
We used a CP solver to solve for plans of such length, and because of the simple structure in the constraints, the solver was able to terminate under 15 seconds in all these instances, outperforming our approach.
	
The next set of experiments replaces the ``until" constraint that \emph{eventually} everyone is on the other bank with the condition that the goal must be satisfied at time $t$, which is a value we vary between test cases.
Because the St-CSP model is modified, requiring different solving times, the range of parameters $(n, b)$ we experimented on is also different.

Table~\ref{Tab:at_fn_MC} shows the experimental results comparing the \verb~@~ operator against \verb~first next~ $\cdots$ \verb~next~.
Each table entry again shows the solving times using the new and old approaches respectively, separated by a ``/", with ``--" denoting a timeout.
The results demonstrate our implementation significantly outperforming the previous approach, with up to 2 orders of magnitude speedup.

\begin{table}[t]
\caption{Missionaries and Cannibals: Time bounded}
\begin{subtable}[t]{.4\textwidth}
	\centering
	\footnotesize
	\vspace*{-.2cm}
	\cprotect\caption{\verb~@~ vs \verb~first next~ $\cdots$ \verb~next~}
	\vspace*{.1cm}
	\begin{tabular}{| c | c | c | c | c | }
		\hline
		$(n,b)$  & $t = 10$ & $t = 40$ & $t = 70$ & $t = 100$ \\
		\hline
		$(20,5)$ & 0.64/49.68 & 4.04/-- & 9.21/-- & 14.84/-- \\
		\hline
		$(30,6)$ & 1.71/178.68 & 16.33/-- & 36.23/-- & 56.76/-- \\
		\hline
		$(40,7)$ & 4.01/454.98 & 38.55/-- & 95.19/-- & 152.79/-- \\
		\hline
		$(50,8)$ & 9.07/-- & 100.34/-- & 236.58/-- & 374.07/-- \\
		\hline
		$(60,9)$ & 17.31/-- & 183.89/-- & 461.51/-- & --/-- \\
		\hline
		$(70,10)$ & 32.25/-- & 371.57/-- & --/-- & --/-- \\
		\hline
	\end{tabular}
	\vspace*{-.2cm}
	\label{Tab:at_fn_MC}
	\end{subtable}
	\hfill
	\begin{subtable}[t]{.4\textwidth}
	\centering
	\footnotesize
	\vspace*{-.2cm}
	\cprotect\caption{CP approach}
	\vspace*{.1cm}
	\begin{tabular}{| c | c | c | c | c | }
		\hline
		$(n,b)$  & $t = 10$ & $t = 40$ & $t = 70$ & $t = 100$ \\
		\hline
		$(20,5)$ & 0.663 & 0.435 &0.562 &1.075  \\
		\hline
		$(30,6)$ & 0.435 &0.560 &0.780 &1.011 \\
		\hline
		$(40,7)$ & 0.562&0.519&0.799 &1.139 \\
		\hline
		$(50,8)$ & 0.762&0.521&0.767&1.102 \\
		\hline
		$(60,9)$ & 1.002&0.501&0.835&0.975 \\
		\hline
		$(70,10)$ & 1.425&0.526&0.873 &0.1109 \\
		\hline
	\end{tabular}
	\vspace*{-.2cm}
	\label{Tab:at_MC_CP}	
	\end{subtable}
\end{table}

For the reader's reference, we also include Table~\ref{Tab:at_MC_CP}, that is the solving time of Gecode finding a single solution/plan for the time-bounded scenario.
Since St-CSP solvers find \emph{all} solutions, it is reasonable to not be competitive with a traditional CP approach.
However, when we asked for \emph{all} solutions instead, \textbf{Gecode timed out} \emph{for all but the $t = 10$ instances}, since the St-CSP search algorithm is able to avoid repeating equivalent search, via dominance detection.
Asking a St-CSP solver to decide only the \emph{existence} of \emph{some} solution, instead of solving for all solutions, is scope for future work.

\subsection{Path Planning in Grid World}

The second set of experiments uses the path finding problem defined by the St-CSP model presented in Fig.~\ref{stcsp for planning}.
We generate random grid worlds of size $n \times n$ by independently sampling each directed edge between adjacent cells with probability $p$, as well as uniformly sampling the start and end points on the grid.
Similarly, we performed two sets of experiments, solving for plans that eventually reach the goal (using the ``until" constraint), and plans that have to reach the goal within a certain number of steps (using the \verb~@~ operator).

For the ``until" experiments, we varied both $n$ and $p$, sampling 50 random instances for each setting of $n$ and $p$.
Fig.~\ref{Fig:RoutingTime} shows the average solving time of the test instances, where instances that timed out count as 600s.
The solving times in this setting increase in $n$ polynomially, and become concave for larger $n$ and $p$ when a substantial number of instances start timing out.

	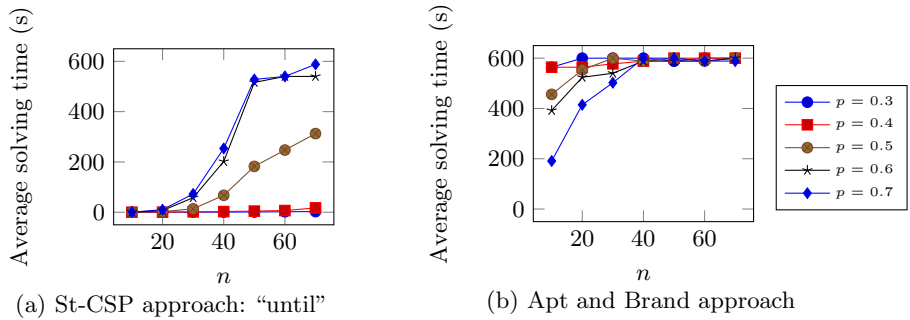
\begin{figure}[t]
		\begin{subfigure}{.48\textwidth}
		\centering
		\vspace*{-.2cm}
		\begin{tikzpicture}
		\begin{axis}[xlabel={$n$},ylabel={Average solving time (s)},legend style={at={(.87,.42)},anchor=center,font=\tiny}, scale only axis, height = .4\textwidth, width = .5\textwidth,label style={font=\small},
		tick label style={font=\small}, ymin=-50, ymax = 650]
		\addplot table[x index=0,y index=1,col sep=comma] {path_planning_until_time.dat};
		\addplot table[x index=0,y index=2,col sep=comma] {path_planning_until_time.dat};
		\addplot table[x index=0,y index=3,col sep=comma] {path_planning_until_time.dat};
		\addplot table[x index=0,y index=4,col sep=comma] {path_planning_until_time.dat};
		\addplot table[x index=0,y index=5,col sep=comma] {path_planning_until_time.dat};
		\end{axis}
		\end{tikzpicture}
		\vspace*{-.2cm}
		\caption{St-CSP approach: ``until"}
		\label{Fig:RoutingTime}
	\end{subfigure}
	\hfill
	\begin{subfigure}{.48\textwidth}
				\centering
				\vspace*{-.3cm}
				\begin{tikzpicture}
				\begin{axis}[xlabel={$n$},ylabel={Average solving time (s)},legend style={at={(1.4,.42)},anchor=center,font=\tiny}, scale only axis, height = .4\textwidth, width = .5\textwidth,label style={font=\small},
		tick label style={font=\small}, ymin=-50, ymax = 650]
				\addplot table[x index=0,y index=1,col sep=comma] {until_time_t1_gecode.dat};
				\addlegendentry{$p = 0.3$}
				\addplot table[x index=0,y index=2,col sep=comma] {until_time_t1_gecode.dat};
				\addlegendentry{$p = 0.4$}
				\addplot table[x index=0,y index=3,col sep=comma] {until_time_t1_gecode.dat};
				\addlegendentry{$p = 0.5$}
				\addplot table[x index=0,y index=4,col sep=comma] {until_time_t1_gecode.dat};
				\addlegendentry{$p = 0.6$}
				\addplot table[x index=0,y index=5,col sep=comma] {until_time_t1_gecode.dat};
				\addlegendentry{$p = 0.7$}
				\end{axis}
				\end{tikzpicture}
				\vspace*{-.6cm}
				\caption{Apt and Brand approach}
				\label{Fig:Routing_Apt}
			\end{subfigure}
\vspace*{-.1cm}
\caption{Path Planning: Eventuality condition}
\vspace*{-.1cm}
\end{figure}

For comparison, Fig.~\ref{Fig:Routing_Apt} shows the solving time using the Apt and Brand~\cite{Apt:2006} framework.
The figures show that most of the instances timed out, demonstrating that the St-CSP approach is far more efficient.
Since any simple path on the grid has an upper bound of $n^2$ in length, similarly to the previous setting we also used a CP solver to solve for a plan of length $n^2$.
However, Gecode runs into memory issues around $n = 40$, exceeding the 256GB memory available.
Even before so, for $n = 10$ a significant proportion of the instances already timed out, even though the St-CSP solves them almost instantaneously (as in Fig.~\ref{Fig:RoutingTime}).
Because of the memory issues that Gecode ran into, we decided to not give corresponding runtime plots since runtime is ill-defined.

For our last set of experiments, we again replace the ``until" constraint with the constraint that the path must have visited the end point by $t$ steps, a parameter that we vary across test cases.
We generated 50 random instances for a selected set of $n$ values, however fixing $p = 0.8$ to make sure that a sizeable portion of the instances are satisfiable.
We further varied $t$ on these instances.

Fig.~\ref{Fig:Grid_at} shows the average solving times by the old and new St-CSP approaches.
We observe a 2 orders of magnitude improvement in solving time for large $t$.
The plots for the \verb~@~ operator are also in general better behaved.
We further found that the reason for the essentially horizontal plots for the ``\verb~first next~ $\cdots$ \verb~next~" operator phrase is due to it only being able to solve the trivially unsatisfiable instances in under 1 second, where the reachable component from the start point is small.
All the other cases timed out, giving the plateau we observe in solving time for the operator phrase.

Fig.~\ref{Fig:Grid_at_CP} shows the solving time using Gecode.
The plots display similar plateauing behaviour as our old appproach, only starting earlier at $t = 20$.
In comparison, the St-CSP approach is competitive with Gecode, despite the St-CSP solver being a prototype.
We believe that it is due to the inherent specification complexity of the path planning problem on the grid.
The entire graph structure has to be encoded for each time point, meaning that for the CP approach, the program is of size $O(tn^2)$, whereas the St-CSP is only of size $O(n^2)$.

\begin{figure}[t]
	\begin{subfigure}{.48\textwidth}
		\centering
		\vspace*{-.3cm}
		\begin{tikzpicture}
		\begin{axis}[xlabel={$t$},ylabel={Average solving time (s)},legend style={at={(.87,.42)},anchor=center,font=\tiny}, scale only axis, height = .4\textwidth, width = .5\textwidth,label style={font=\small},
		tick label style={font=\small}, ymin=-50, ymax = 650]
		\addplot[blue, mark=*] table[x index=0,y index=1,col sep=comma] {path_planning_at_time.dat};
		\addplot[red, mark=square*] table[x index=0,y index=2,col sep=comma] {path_planning_at_time.dat};
		\addplot[teal, mark = triangle*] table[x index=0,y index=3,col sep=comma] {path_planning_at_time.dat};
		\addplot[black, mark=star] table[x index=0,y index=4,col sep=comma] {path_planning_at_time.dat};
		\addplot[violet, mark = diamond*] table[x index=0,y index=5,col sep=comma] {path_planning_at_time.dat};
		\addplot[blue, mark=*, dashed]  table[x index=0,y index=6,col sep=comma] {path_planning_at_time.dat};
		\addplot[red, mark=square*, dashed] table[x index=0,y index=7,col sep=comma] {path_planning_at_time.dat};
		\addplot[teal,mark = triangle*, dashed] table[x index=0,y index=8,col sep=comma] {path_planning_at_time.dat};
		\addplot[black, mark=star,  dashed] table[x index=0,y index=9,col sep=comma] {path_planning_at_time.dat};
		\addplot[violet, mark = diamond*, dashed]table[x index=0,y index=10,col sep=comma] {path_planning_at_time.dat};
		\end{axis}
		\end{tikzpicture}
		\vspace*{-.2cm}
		\cprotect\caption{\verb~@~ (solid)/\verb~first next~ (dashed)}
		\label{Fig:Grid_at}
	\end{subfigure}
	\begin{subfigure}{.48\textwidth}
			\centering
			\vspace*{-.3cm}
			\begin{tikzpicture}
			\begin{axis}[xlabel={$t$},ylabel={Average solving time (s)},legend style={at={(1.4,.42)},anchor=center,font=\tiny}, scale only axis, height = .4\textwidth, width = .5\textwidth,label style={font=\small},
		tick label style={font=\small}, ymin=-50, ymax = 650]
			\addplot[blue, mark=*] table[x index=0,y index=1,col sep=comma] {path_planning_at_time_CP.dat};
			\addlegendentry{$n=5$}
			\addplot[red, mark=square*] table[x index=0,y index=2,col sep=comma] {path_planning_at_time_CP.dat};
			\addlegendentry{$n=10$}
			\addplot[teal, mark = triangle*] table[x index=0,y index=3,col sep=comma] {path_planning_at_time_CP.dat};
			\addlegendentry{$n=15$}
			\addplot[black, mark=star] table[x index=0,y index=4,col sep=comma] {path_planning_at_time_CP.dat};
			\addlegendentry{$n=20$}
			\addplot[violet, mark = diamond*] table[x index=0,y index=5,col sep=comma] {path_planning_at_time_CP.dat};
			\addlegendentry{$n=25$}
			\end{axis}
			\end{tikzpicture}
			\vspace*{-.6cm}
			\caption{CP approach}
			\label{Fig:Grid_at_CP}
		\end{subfigure}
\vspace*{-.1cm}
\caption{Path planning: Time bounded}
\vspace*{-.1cm}
\end{figure}
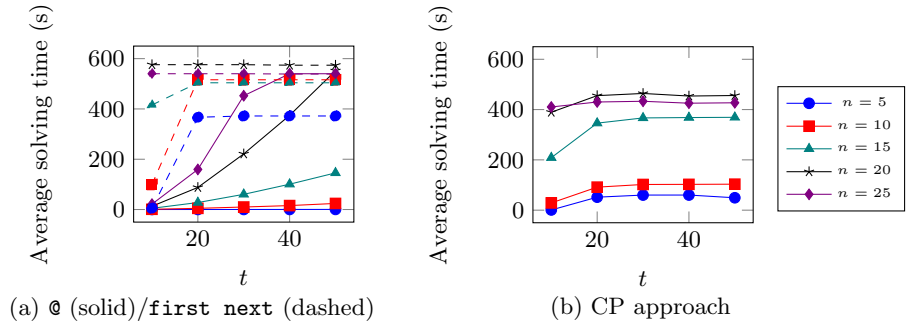
		
\section{Concluding Remarks}
\label{conclusion}
\vspace*{-.1cm}

Our work improves the expressiveness of the St-CSP framework by
augmenting it with 1) the new ``until" constraint construct, adapted from
the corresponding LTL operator, and 2) the \verb~@~ operator, which
is a syntactic sugar for \verb~first next~ $\cdots$ \verb~next~ that
further allows for faster solving by exploiting the special structure of
the expression.  We give corresponding new St-CSP solving algorithms, and
also experimental evidence for their competitiveness with the corresponding
CP approaches using Gecode.  In our opinion the @ operator and the
``until'' constraint are for different purposes. The former is for time
bounded scenario, while the latter is useful, for example, from a security
perspective: we wish to know that our adversary can never achieve a
sinister goal regardless of time budget.

By introducing the ``until" constraint, we altered the structure of the generated solution automata and the guarantee we give regarding the execution of the automata (Section~\ref{Sect:Prune}).
From the statement that every run of the automaton is an accepting run, we weaken the guarantee (\emph{whilst maintaining practical relevance}) to such that every finite run of the automaton could be extended to an infinite length solution stream.
A natural direction for further investigation is to consider, under this weaker guarantee, how much more expressive can the St-CSP framework become.
Are there other practical and natural constraints or temporal operators that, despite being currently inexpressible in the St-CSP framework, can be introduced with a solving algorithm that provides the above guarantee?
Can we identify even weaker, yet still practically relevant guarantees that allows for even more expressiveness in the framework?
We leave the answering of these questions for future work.

\bibliographystyle{splncs03}
\bibliography{database}

\appendix

\section{Constraint Model for the Missionaries and Cannibals Experiment}
\label{Appx:Missionaries}

For the modelling part of our experimentation, our approach was to first write St-CSPs describing the problem, and then from the St-CSPs generate the finite domain CSPs that we use for comparison.
The finite domain CSP models are essentially the same as the St-CSP models, except that the individual time points are ``unrolled" such that each stream variable at each time point corresponds to a single finite domain variable.
As such, in this appendix we only present the St-CSP model of the experiment as Figure~\ref{Fig:MC_StCSP} on the next page.
The model, as explained before, is parametrised by (1) $n$: the number of missionaries and also the same number of cannibals, as well as (2) $b$: the capacity of the boat.

\begin{figure}
\begin{tabular}{l}
var $leftmissionaries$ with alphabet $[0..n]$ // number of missionaries on the left bank\\
var $rightmissionaries$ with alphabet $[0..n]$ // number of missionaries on the right bank\\
var $leftcannibals$ with alphabet $[0..n]$ // number of cannibals on the left bank\\
var $rightcannibals$ with alphabet $[0..n]$ // number of cannibals on the right bank\\
var $boat$ with alphabet $[0..1]$ // the direction of the boat\\
var $succ$ with alphabet $[0..1]$ // whether the success condition is achieved\\
\\

\verb~first~ $leftmissionaries$ \verb~==~ $n$\\
\verb~first~ $leftcannibals$ \verb~==~ $n$\\
\verb~first~ $rightmissionaries$ \verb~==~ 0\\
\verb~first~ $rightcannibals$ \verb~==~ 0\\
\verb~first~ $boat$ \verb~==~ 0\\
\\

On each bank, if there are missionaries, then they cannot be outnumbered by cannibals:\\
$leftcannibals$ \verb~<=~ \verb~if~ $leftmissionaries$ \verb~eq~ 0 \verb~then~ 4 \verb~else~ $leftmissionaries$\\
$rightcannibals$ \verb~<=~ \verb~if~ $rightmissionaries$ \verb~eq~ 0 \verb~then~ 4 \verb~else~ $rightmissionaries$\\
\\

The boat needs at least 1 person until we finish the game:\\
\verb~abs~($leftmissionaries$ \verb~-~ \verb~next~ $leftmissionaries$)\\
\quad \quad \quad \verb~+~ \verb~abs~($leftcannibals$ \verb~-~ \verb~next~ $leftcannibals$) \verb~>=~ \verb~if~ $succ$ \verb~then~ 0 \verb~else~ 1\\
\\

The boat has capacity $b$:\\
\verb~abs~($leftmissionaries$ \verb~-~ \verb~next~ $leftmissionaries$)\\
\quad \quad \quad \verb~+~ \verb~abs~($leftcannibals$ \verb~-~ \verb~next~ $leftcannibals$) \verb~<=~ $b$\\
\\

Conservation of mass:\\
$leftmissionaries$ \verb~-~ \verb~next~ $leftmissionaries$ \verb~==~ \verb~next~ $rightmissionaries$ \verb~-~ $rightmissionaries$\\
$leftcannibals$ \verb~-~ \verb~next~ $leftcannibals$ \verb~==~ \verb~next~ $rightcannibals$ \verb~-~ $rightcannibals$\\
\\

The direction of the boat determines the increase and decrease of numbers on each bank:\\
$boat$ \verb~eq~ 1 \verb~<=~ ($leftmissionaries$ \verb~-~ \verb~next~ $leftmissionaries$) \verb~le~ 0\\
$boat$ \verb~eq~ 1 \verb~<=~ ($leftcannibals$ \verb~-~ \verb~next~ $leftcannibals$) \verb~le~ 0\\
$boat$ \verb~eq~ 0 \verb~<=~ ($leftmissionaries$ \verb~-~ \verb~next~ $leftmissionaries$) \verb~ge~ 0\\
$boat$ \verb~eq~ 0 \verb~<=~ ($leftcannibals$ \verb~-~ \verb~next~ $leftcannibals$) \verb~ge~ 0\\
\\

The direction of the boat always alternates until we finish the game:\\
\verb~next~ $boat$ \verb~==~ \verb~if~ $succ$ \verb~then~ $boat$ \verb~else~ \verb~if~ $boat$ \verb~eq~ 1 \verb~then~ 0 \verb~else~ 1\\
\\

We finish the game when everyone is on the other bank:\\
$succ$ \verb~==~ $rightmissionaries$ \verb~eq~ $n$ \verb~and~ $rightcannibals$ \verb~eq~ $n$\\
\\

We stop moving people once we have succeeded:\\
$succ$ \verb~<=~ (\verb~next~ $leftmissionaries$) \verb~eq~ $leftmissionaries$\\
$succ$ \verb~<=~ (\verb~next~ $leftcannibals$) \verb~eq~ $leftcannibals$\\
\\

Depending on the experiment, we use either the eventuality condition for success:\\
1 \verb~until~ $succ$\\
\\

Or we enforce that we finish the game by the $t^\text{th}$ step:\\
$succ$ \verb~@~ $t$ \verb~==~ 1
\end{tabular}

	\caption{St-CSP Model for Missionaries and Cannibals}
	\label{Fig:MC_StCSP}
\end{figure}

\end{document}